\documentclass{article}
\usepackage{ijcai19}

\usepackage{times}
\usepackage{soul}
\usepackage{url}
\usepackage[hidelinks]{hyperref}
\usepackage[utf8]{inputenc}
\usepackage[small]{caption}
\usepackage{graphicx}
\usepackage{amsmath}
\usepackage{booktabs}
\usepackage{enumitem}
\usepackage{amsmath}
\usepackage{float}
\usepackage{subfig} %
\usepackage{subfloat}
\usepackage{amssymb,amsmath,amsthm}
\usepackage{multicol}
\usepackage{bbm}
\usepackage{microtype}
\usepackage{booktabs} 
\usepackage{algorithm}
\usepackage{algorithmic}
\usepackage{xcolor}
\usepackage{multirow}

\newcommand{\D}{\mathcal{D}}
\newcommand{\Proba}{\mathbb{P}}
\newcommand{\R}{\mathbb{R}}

\newcommand{\E}{\mathbb{E}}

\newcommand{\x}{\mathbf{x}}

\newcommand{\y}{\mathbf{y}}

\newcommand{\balpha}{\boldsymbol{\alpha}}
\newcommand{\btheta}{\boldsymbol{\theta}}
\newcommand{\calH}{\mathcal{H}}
\newcommand{\calX}{\mathcal{X}}

\newtheorem{theorem}{Theorem}
\newtheorem*{theorem*}{Theorem}
\newtheorem{lemma}{Lemma}

\title{A Principled Approach for Learning Task Similarity in Multitask Learning}

\author{
Changjian Shui$^1$\and
Mahdieh Abbasi$^1$\and
Louis-Émile Robitaille$^{1}$\and
Boyu Wang$^2$\and
Christian Gagn\'e$^{1}$
\\
\affiliations
$^1$Universit\'e Laval, 
$^2$ University of Pennsylvania\\
\emails
\{changjian.shui.1, mahdieh.abbasi.1,louis-emile.robitaille.1\}@ulaval.ca,\\ boyuwang@seas.upenn.edu, 
christian.gagne@gel.ulaval.ca
}

\begin{document}

\maketitle

\begin{abstract}
Multitask learning aims at solving a set of related tasks simultaneously, by exploiting the shared knowledge for improving the performance on individual tasks. Hence, an important aspect of multitask learning is to understand the similarities within a set of tasks. Previous works have incorporated this similarity information explicitly (e.g., weighted loss for each task) or implicitly (e.g., adversarial loss for feature adaptation), to achieve good empirical performances. However, the theoretical motivations for adding task similarity knowledge are often missing or incomplete. In this paper, we give a different perspective from a theoretical point of view to understand this practice. We first provide an upper bound on the generalization error of multitask learning, showing the benefit of explicit and implicit task similarity knowledge. We systematically derive the bounds based on two distinct task similarity metrics: $\calH$ divergence and Wasserstein distance. From these theoretical results, we revisit the Adversarial Multitask Neural Network, proposing a new training algorithm to learn the task relation coefficients and neural network parameters iteratively. We assess our new algorithm empirically on several benchmarks, showing not only that we find interesting and robust task relations, but that the proposed approach outperforms the baselines, reaffirming the benefits of theoretical insight in algorithm design.
\end{abstract}

\section{Introduction}
Traditional machine learning mainly focused on designing learning algorithms for individual problems. While significant progress has been achieved in applied and theoretical research, it still requires a large amount of labelled data in such a context to obtain a small generalization error. In practice, this can be highly prohibitive, e.g., for modelling users’ preferences for products \cite{murugesan2017active}, for classifying multiple objects in computer vision \cite{long2017learning}, or for analyzing patient data in computational healthcare \cite{wang2015online}. In the multitask learning (MTL) scenario, an agent learns the shared knowledge between a set of related tasks. Under different assumptions on task relations, MTL has been shown to reduce the amount of labelled examples required per task to reach an acceptable performance.       

Understanding the theoretical assumptions of the tasks relationship plays a key role in designing a good MTL algorithm. In fact, it determines which \emph{inductive bias} should be involved in the learning procedure. Recently, there are many successful algorithms that rely on task similarity information, which assumes the \emph{Probabilistic Lipschitzness} (PL) condition \cite{urner2013probabilistic} as the inductive bias. For instance, \cite{murugesan2017active,murugesan2016adaptive,pentina2017multi} minimize a weighted sum of empirical loss in which similar tasks are assigned higher weights. These approaches explicitly estimate the task similarities through a linear model. Since these approaches are estimated in the original input space, it is difficult to handle the \emph{covariate shift} problem. Therefore, many neural network based approaches started to explore tasks similarities implicitly: \cite{liu2017adversarial,li2018extracting} use adversarial losses by feature adaptation, minimizing the distribution distance between the tasks to construct a shared feature space. Then, the hypothesis for the different tasks are learned over this adapted feature space.  

The implicit similarity learning approaches are inspired from the idea of Generative Adversarial Networks (GANs) \cite{goodfellow2014generative}. However, the fundamental implications of incorporating task similarity information in MTL algorithms are not clear. The two main questions are \emph{why} should we combine explicit and implicit similarity knowledge in the MTL framework and \emph{how} can we properly do it.        

Previous works either consider explicit or implicit similarity knowledge separately, or combine them heuristically in some specific applications. In contrast, the main goal of our work is to give a rigorous analysis of the benefits of task similarities and derive an algorithm which properly uses this information. We start by deriving an upper bound on the generalization error of MTL under different similarity metrics (or adversarial loss). These bounds show the motivation behind the use of adversarial loss in MTL, that is to control the generalization error. Then, we derive a new procedure to update the relationship coefficients from these theoretical guarantees. This procedure allows us to bridge the gap between the explicit and implicit similarities, which have been previously seen as disjointed or treated heuristically. We then derive a new algorithm to train the Adversarial Multitask Neural Network (AMTNN) and validate it empirically on two benchmarks: digit recognition and Amazon sentiment analysis. The results show that our method not only highlights some interesting relations, but also outperforms the previous baselines, reaffirming the benefits of theory in algorithm design.

\section{Related Work}

\paragraph{Multitask learning (MTL)} A broad and detailed presentation of the general MTL is provided in some survey papers \cite{zhang2017survey,ruder2017overview}. More specifically related to our work, we note several practical approaches that use tasks relationship to improve empirical performances: \cite{zhang2010convex} solve a convex optimization problem to estimate tasks relationships, while \cite{long2017learning,kendall2018multi} propose probabilistic models through construction of a task covariance matrix or estimate the multitask likelihood from a deep Bayes model. On the theoretical side, \cite{murugesan2016adaptive,murugesan2017active,pentina2017multi} analyze the weighted sum loss algorithm and its applications in the online, active and transductive scenarios. Moreover, \cite{maurer2016benefit} analyze the generalization error of representation-based approaches, and \cite{zhang2015multi} analyze the algorithmic stability in MTL.

\paragraph{Similarity metrics and adversarial loss} The \emph{similarity metric} (or distribution divergence) is widely used in deep generative models \cite{goodfellow2014generative,arjovsky2017wasserstein}, domain adaptation \cite{ben2010theory,ganin2016domain,redko2017theoretical}, robust learning \cite{konstantinov2019robust}, and meta-learning \cite{rakotomamonjy2018wasserstein}. In transfer learning, adversarial losses are widely used for feature adaptation, since the transfer procedure is much more efficient on a shared representation. In applied transfer learning, $\calH$ divergence \cite{ganin2016domain} and Wasserstein distance \cite{li2018extracting} are widely used in the divergence metric. 
As for MTL applications, \cite{liu2017adversarial} and \cite{kremer2018inductive} apply $\calH$-divergence in natural language processing for text classification and speech recognition, while \cite{janati2018wasserstein} are the first to propose the use of Wasserstein distance to estimate the similarity of linear parameters instead of the data generation distributions. As for the theoretical understanding, \cite{lee2018minimax} analyzes the minimax statistical property in the Wasserstein distance.

\section{Preliminaries}
Considering a set of $T$ tasks $\{\hat{\D}_t\}_{t=1}^T$, in which the observations are generated by the underlying distribution $\D_{t}$ over $\mathcal{X}$ and the real target is determined by the underlying labelling functions $f_{t}:\mathcal{X}\to\mathcal{Y}$ for $\{(\D_t,f_t)\}_{t=1}^T$. Then, the goal of MTL is to find $T$ hypothesis: $h_1,\dots,h_T$ over the hypothesis space $\mathcal{H}$ to control the average expected error of all the tasks: 
\begin{equation*}
\frac{1}{T} \sum_{t=1}^T R_t(h_t),
\end{equation*}
where $R_t(h_t) \equiv R_t(h_t,f_t) = \E_{\x\sim\D_t} \ell(h_t(\x),f_t(\x))$ is the expected risk at task $t$ and $\ell$ is the loss function. Throughout the theoretical part, the loss is $\ell(h_t(\x),f_t(\x)) = |h_t(\x)-f_t(\x)|$, which is coherent with \cite{pentina2017multi,li2018extracting,ben2010theory,ganin2016domain,redko2017theoretical}. If $h,f$ are the binary mappings with output in $\{-1,1\}$, it recovers the typical zero-one loss. 

We also assume that each task has $m_t$ examples, with $\sum_{t=1}^T m_t = m$ examples in total. Then for each task $t$, we consider a minimization of weighted empirical loss for each task. That means we define a simplex $\balpha_t\in \Delta^T = \{\balpha_{t,i}\geq 0,~ \sum_{i=1}^T \balpha_{t,i} = 1 \}$ for the corresponding weight for task $t$. Then the weighted empirical error w.r.t. the hypothesis $h$ for task $t$ can be written as:
\begin{equation*}
\hat{R}_{\balpha_t}(h) = \sum_{i=1}^T \balpha_{t,i} \hat{R}_{i}(h),
\end{equation*}
where $\hat{R}_{i}(h) = \frac{1}{m_i}\sum_{j=1}^{m_i} \ell(h(\x_j),\y_j) $ is the average empirical error for task $i$. 

\section{Similarity Measures}
As we illustrated in the previous section, we are interested in task similarity metrics in MTL. Therefore, the first element to determine is how to measure the similarity between two distributions. For this, we introduce two metrics: $\mathcal{H}$-divergence \cite{ben2010theory} and Wasserstein distance \cite{arjovsky2017wasserstein}, which are widely used in machine learning.  

\paragraph{$\calH$-divergence}
Given an input space $\mathcal{X}$ and two probability distributions $\D_i$ and $\D_{j}$ over $\mathcal{X}$, let $\mathcal{H}$ be a hypothesis class on $\mathcal{X}$. We define the $\calH$-divergence of two distributions as
\begin{equation*}
	d_{\calH} (\D_i,\D_j) =  \sup_{h,h^{\prime}\in\calH} |R_{i}(h,h^{\prime})-R_{j}(h,h^{\prime})|.
\end{equation*}
The empirical $\calH$-divergence corresponds to: 
\begin{equation*}
    d_{\calH}(\hat{\D}_i,\hat{\D}_j) =  \sup_{h,h^{\prime}\in\calH} |\hat{R}_{i}(h,h^{\prime})-\hat{R}_{j}(h,h^{\prime})|.
\end{equation*}

\paragraph{Wasserstein distance}
We assume $\calX$ is the measurable space and denote $\mathcal{P}(\calX)$ as the set of all probability measures over $\calX$. Given two probability measures $\D_i\in\mathcal{P}(\calX_1)$ and $\D_j \in\mathcal{P}(\calX_2)$, the \emph{optimal transport} (or Monge-Kantorovich) problem can be defined as searching for a probabilistic coupling $\gamma$ refined as a joint probability measure over $\calX_1 \times \calX_2$  with marginals $\D_i$ and $\D_j$ for all $\x,\y$ that are minimizing the cost of transport w.r.t. some cost function $c$:
\begin{align*}
& \mathrm{argmin}_{\gamma} \int_{\calX_1 \times \calX_2} c(\x,\y)^p d\gamma(\x,\y),\\
& \mathrm{s.t.}\quad \mathbf{P}^{\calX_1} \#\gamma = \D_i;\quad \mathbf{P}^{\calX_2} \#\gamma = \D_j,
\end{align*}
where $\mathbf{P}^{\calX_1}$ is the projection over $\calX_1$ and $\#$ denotes the push-forward measure. The Wasserstein distance of order $p$ between $\D_i$ and $\D_j$ for any $p \geq 1$ is defined as:
\begin{equation*}
W_p^p(\D_i,\D_j) = \inf_{\gamma\in \Pi(\D_i,\D_j)} \int_{\calX_1 \times \calX_2} c(\x,\y)^p d \gamma(\x,\y),
\end{equation*}
where $c:\calX\times\calX \to \R^{+}$ is the cost function of transportation of one unit of mass $\x$ to $\y$ and $\Pi(\D_i,\D_j)$ is the collection of all joint probability measures on $\calX \times \calX$ with marginals $\D_i$ and $\D_j$. Throughout this paper, we only consider the case of $p=1$, i.e., Wasserstein-1 distance.

\section{Theoretical Guarantees}
Based on the definitions of the distribution similarity metric, we are demonstrating that the generalization error in the multitask learning can be upper bounded by the following result:
\begin{theorem}
Let $\calH$ be a hypothesis family with a VC-dimension $d$. If we have $T$ tasks generated by the underlying distribution and labelling function $\{(\D_1,f_1),\ldots,(\D_T,f_T)\}$ with observation numbers $m_1,\dots,m_T$. If we adopt the $\mathcal{H}$ divergence as a similarity metric, then for any simplex $\balpha_{t}\in\Delta^T$, and for $\delta\in(0,1)$, with probability at least $1-\delta$, for $h_1,\dots,h_T \in\mathcal{H}$, we have:
\begin{align*}
    & \frac{1}{T}\sum_{t=1}^T R_t(h_t)  \leq  \underbrace{\frac{1}{T}\sum_{t=1}^T \hat{R}_{\balpha_t}(h_t)}_\text{Weighted empirical loss}+ \underbrace{C_1  \sum_{t=1}^T \Big( \sqrt{ \sum_{i=1}^T \frac{\balpha^2_{t,i}}{\beta_i}} \Big)}_\text{Coefficient regularization}\\
    & + \underbrace{\frac{1}{T}\sum_{t=1}^T \sum_{i=1}^T \balpha_{t,i} d_{\calH}(\hat{\D}_t,\hat{\D}_i)}_\text{Empirical distribution distance} + \underbrace{C_2 + \frac{1}{T}\sum_{t=1}^T\sum_{i=1}^T \balpha_{t,i}\lambda_{t,i}}_\text{Complexity \& optimal expected loss},
\end{align*}
where $\beta_i = \frac{m_i}{m}$, $C_1 = 2 \sqrt{ \frac{2(d\log(\frac{2em}{d})+\log(\frac{16T}{\delta}))}{m}}$ and $C_2 =2\min_{i,j} \sqrt{\frac{2d\log(2m_{i,j})+ \log(32T/\delta)}{m_{i,j}}}$ with $m_{i,j} = \min\{ m_i,m_j\}$ and $\lambda_{i,j} = \inf_{h\in\calH} \{R_i(h) + R_j(h)\}$ (joint expected minimal error w.r.t. $\calH$).
\end{theorem}

Theorem 1 illustrates that the upper bound on the generalization error in our MTL settings can be decomposed into the following terms:
\begin{itemize}
    \item The empirical loss and empirical distribution similarities control the weights (or task relation coefficient) $\balpha_1,\dots,\balpha_T$. For instance, for a given task $t$, if task $i$ has a small empirical distance $d_{\calH}(\hat{\D}_t,\hat{\D}_i)$ and hypothesis $h_t$ has a small empirical loss $\hat{R}_i(h_t)$ on task $i$, it means that task $i$ is very similar to $t$. Hence, more information should be borrowed from task $i$ when learning $t$ and the corresponding coefficient $\balpha_{t,i}$ should have high values.
    \item Simultaneously, the \emph{coefficient regularization term} prevents the relation coefficients locating only on the $\balpha_{t,t}$, in which it will completely recover the independent MTL framework. Then the coefficient regularization term proposed a trade-off between learning the single task and sharing information from the others tasks. 
    \item The complexity and optimal terms depend on the setting hypothesis family $\mathcal{H}$. Given a fixed hypothesis family such as neural network, the complexity is constant. As for the optimal expected loss, throughout this paper we assume $\lambda_{t,i}$ is \emph{much smaller} than the empirical term, which indicates that the hypothesis family $\calH$ can learn the multiple tasks with a small expected risk. This is a natural setting in the MTL problem since we want the predefined hypothesis family to learn well for all of the tasks. While a high expected risk means such a hypothesis set cannot perform well, which contradicts our assumption.
\end{itemize}
    
In Theorem 1, we have derived a bound based on the $\calH$ divergence and applied in the classification problem. Then we proposed another bound based on the Wasserstein distance, which can be applied in the classification and regression problem.

\begin{theorem}
Let $\calH$ be a hypothesis family from $\calX$ to $[0,1]$, with pseudo-dimension $d$ and each member $h\in\calH$ is $K$ Lipschitz. If we have $T$ tasks generated by the underlying distribution and labelling function $\{(\D_1,f_1),\dots, (\D_T,f_T)\}$ with observation numbers $m_1,\dots,m_T$. If we adopt Wasserstein-1 distance as a similarity metric with cost function $c(\x,\y)= \|\x-\y\|_2$, then for any simplex $\balpha_{t}\in\Delta^{T}$, and for $\delta\in(0,1)$, with a probability at least $1-\delta$, for $h_1,\dots,h_T \in\mathcal{H}$, we have:
\begin{align*}
& \frac{1}{T}\sum_{t=1}^T R_t(h_t) \leq  \underbrace{\frac{1}{T}\sum_{t=1}^T \hat{R}_{\balpha_t}(h_t)}_\text{Weighted empirical loss} +  \underbrace{C_1  \sum_{t=1}^T \Big( \sqrt{ \sum_{j=1}^T \frac{\balpha^2_{t,j}}{\beta_j}} \Big)}_\text{Coefficient regularization} \\
& +  \underbrace{\frac{2K}{T}\sum_{t=1}^T \sum_{i=1}^T \balpha_{t,i} W_1 (\hat{D}_t,\hat{D}_i)}_\text{Empirical distribution distance} + \underbrace{C_2 +\frac{1}{T}\sum_{t=1}^T\sum_{i=1}^T \balpha_{t,i}\lambda_{t,i}}_\text{Complexity \& optimal expected loss},
\label{wbound}
\end{align*}
where $\beta_i = \frac{m_i}{m}$, $C_1 = 2 \sqrt{ \frac{2(d\log(\frac{2em}{d})+\log(\frac{16T}{\delta}))}{m}}$, $C_2 = \frac{2K}{T} \sum_{t=1}^T \sum_{i=1}^T \gamma_{t,i}$ with $\gamma_{t,i} = \mu_t m_t^{-1/s} + \mu_i m_i^{-1/s} +  \sqrt{\log(\frac{2T}{\delta})}(\sqrt{\frac{1}{m_t}} + \sqrt{\frac{1}{m_i}})$ and $s$ and $\mu_{\cdot}$ are some specified constants.  
\end{theorem}
The proof w.r.t. the Wasserstein-1 distance is analogous to the proof in the $\calH$-divergence but with different assumptions.

\paragraph{Remark} The upper bound of the generalization error shows some intuitions that we should not only minimize the weighted empirical loss, but also minimize the empirical distribution divergence between each task. Moreover, in MTL approaches based on neural networks, these conclusions proposed a theoretical support for understanding the role of \emph{adversarial losses}, which exactly minimize the distribution divergence.  

\section{Adversarial Multitask Neural Network}

From the generalization error upper bound in the MTL framework, we developed a new training algorithm for the Adversarial Multitask Neural Network (AMTNN). It consists in multiple training steps by iteratively  optimizing the parameters in the neural network for a given fixed relation coefficient $\balpha_1,\dots,\balpha_T$ and estimating the relation coefficient, given fixed neural network weights. 

Moreover, we have three types of parameters in AMTNN: $\btheta^{f}$, $\btheta_{\cdot}^{d}$ and $\btheta_{\cdot}^{h}$, corresponding to the parameter for feature extractor, adversarial loss (distribution similarity) and task loss, respectively.   

To simplify the problem, we assume that each task has the same number of observations, i.e., $\beta_i =\frac{1}{T}$, and that regularization will use the $l_2$ norm of $\balpha_t$.

\subsection{Neural network parameters update}

Given a fixed $\balpha_1,\dots,\balpha_T$, according to the theoretical bound, we want to minimize the weighted empirical error $\frac{1}{T}\sum_{t=1}^{T} \hat{R}_{\balpha_t}(\btheta^f,\btheta_{t}^{h})$ and the empirical distribution ``distance'' $\hat{d}(\D_t,\D_i)$ with $t,i=1,\dots,T$. Inspired by \cite{ganin2016domain}, the minimization of the distribution ``distance'' is equivalent to the maximization of the adversarial loss $\hat{E}_{t,i}(\btheta^f,\btheta_{t,i}^d)$. Overall, we have the following loss function with a trade-off coefficient $\rho$:
\begin{equation}
    \min_{\btheta^f,\btheta^h_1,\dots,\btheta^h_t} \max_{\btheta^{d}_{t,i}}  \sum_{t=1}^T \hat{R}_{\balpha_t}(\btheta^f,\btheta_{t}^{h}) +  \rho \sum_{i,t=1}^{T} \balpha_{t,i} \hat{E}_{t,i}(\btheta^{f},\btheta_{t,i}^d).
    \label{nn_loss}
\end{equation}
It should be noted that for a given task $t$, the sum loss can be expressed as $\frac{1}{T}\sum_{i=1}^T \balpha_{t,i} \sum_{\x \in \hat{D}_i} \ell((\x,y),\btheta^f,\btheta_t^h)$, with $\ell$ being the cross entropy loss. This means that the empirical loss is a weighted sum of all of the task losses, determined by task relation coefficient $\balpha_t$. This is coherent with \cite{murugesan2016adaptive}, which does not provide theoretical explanations. 

Also, the adversarial loss $\hat{E}_{t,i}(\btheta^f,\btheta^d_{t,i})$ is a symmetric metric for which we need to compute $\hat{E}_{t,i}$ only for $t<i$. Motivated by \cite{ganin2016domain}, the neural network will output for a pair of observed \emph{unlabeled} tasks $(\hat{\D}_t,\hat{\D}_i)$ a score in $[0,1]$ to predict from which distribution it comes. Supposing the output function is $g_{t,i}(\x,(\btheta^f,\btheta^d_{t,i}))\equiv g_{t,i}(\x)$, the adversarial loss will be the following under different distance metrics:
\begin{description}
\item[$\calH$ divergence:] $$\hat{E}_{t,i} = \sum_{\x\in\hat{\D}_t} \log(g_{t,i}(\x)) + \sum_{\x\in\hat{D}_i} \log(1-g_{t,i}(\x));$$
\item[Wasserstein-1 distance:] Since the primal form has a high computational complexity, we adopted the same strategy as \cite{arjovsky2017wasserstein} by estimating the empirical Kantorovich-Rubinstein duality of  Wasserstein-1 distance, which leads to $W_1(\hat{D}_t,\hat{D}_i) = \frac{1}{K} \sup_{\|f\|\leq K} \Big( \E_{\x\in\hat{\D}_t} [g_{t,i}(\x)] - \E_{\x\in\hat{\D}_i}[g_{t,i}(\x)] \Big)$. Combining it with the result of Theorem 2, we can derive that $$\hat{E}_{t,i} = \E_{\x\in\hat{\D}_t} [g_{t,i}(\x)] - \E_{\x\in\hat{\D}_i}[g_{t,i}(\x)].$$
\end{description}

\begin{figure}[t!]
\centering 
	\includegraphics[width=0.35\textwidth]{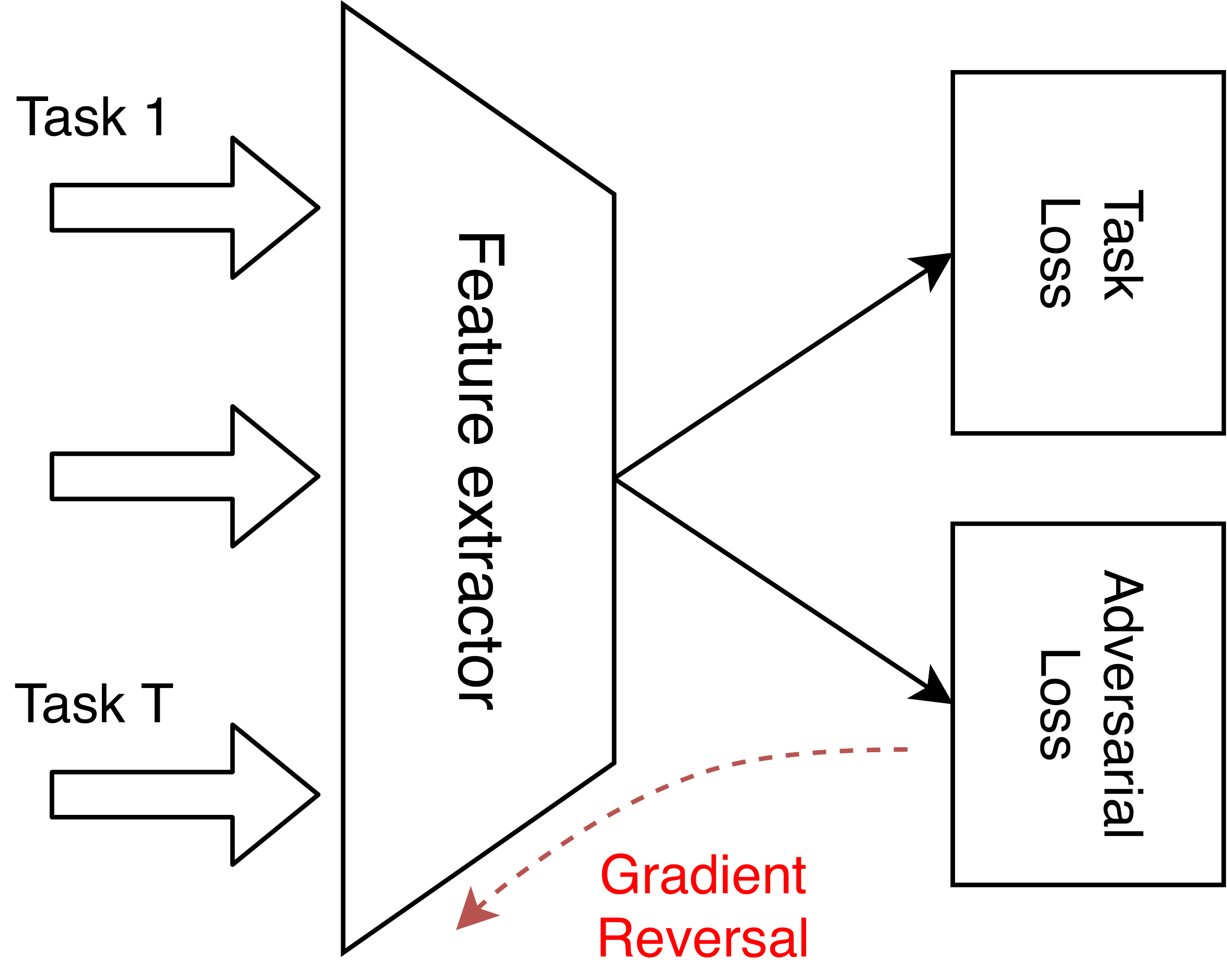}
	\caption{General framework of Adversarial Multitask Neural Network (AMTNN).}
\label{fig:deep_MTL}
\end{figure}

\begin{center}
\begin{algorithm}[t!]
		\caption{AMTNN updating algorithm}
		\begin{algorithmic}[1] 
		\REQUIRE Samples from different tasks $\{\hat{\D}_t\}_{t=1}^T$, initial coefficients $\{\balpha_t\}_{t=1}^T$ and learning rate $\eta$
        \ENSURE Neural network parameters $\btheta^{f}$, $\btheta_{\cdot}^h$, $\btheta_{\cdot}^d$ and relationship coefficient $\balpha_1,\dots,\balpha_T$
        \FOR{mini-batch of samples $\{(\x^b_t,\y^b_t)\}$ from $\{\hat{\D}_t\}_{t=1}^T$}
        \STATE For each the distribution pair $(t,i)$ with $t<i$, compute the adversarial loss $\hat{E}_{t,i}(\btheta^f,\btheta^d_{t,i})$.
        \STATE For each task $t$, define the empirical loss matrix $\hat{R}_{t,i} = \sum_{(\x_{i}^b,\y_{i}^b) \in \hat{D}_i} \ell((\x_{i}^b,\y_{i}^b),\btheta^f,\btheta_t^h) $ and compute the label loss:
        $$\hat{R}_{\balpha_t}  = \sum_{i=1}^T \balpha_{t,i}\hat{R}_{t,i} $$
        \STATE Update $\btheta^f,\btheta^h_{t}$: $\btheta^h_{t} = \btheta^h_{t} - \eta \frac{\partial \hat{R}_{\balpha_t}}{\partial \btheta^h_{t}} $ and $\btheta^f = \btheta^f - \eta \left(\sum_{t=1}^T \frac{\partial \hat{R}_{\balpha_t}}{\partial \btheta^f} + \sum_{t,i: t<i}^T (\balpha_{t,i }+\balpha_{i,t})\frac{\partial \hat{E}_{t,i}}{\partial \btheta^f}\right) $
        \STATE Update $\btheta^d_{t,i}$ ($t<i$): $\btheta_{t,i}^d = \btheta_{t,i}^d + \eta \Big( (\balpha_{t,i}+\balpha_{i,t})\frac{\partial \hat{E}_{t,i}}{\partial \btheta^d_{t,i}}) \Big)$
        \ENDFOR
		\STATE Re-estimate $\{\balpha_t\}_{t=1}^T$ by optimizing over Eq.~(\ref{alpha_loss}).
        \end{algorithmic}
        \label{amtnn_algo}
\end{algorithm}
\end{center}

\subsection{Relation coefficient updating}

The second step after updating the neural network parameter, we need to re-estimate the coefficients $\balpha_1,\dots,\balpha_T$ when giving fixed $\btheta^f,\btheta^h_{\cdot},\btheta^d_{\cdot}$. According to the theoretical guarantees, we need to solve the following convex constraint optimization problem:
\begin{align}
        \min_{\balpha_1,\dots,\balpha_T} & \sum_{t=1}^T \hat{R}_{\balpha_t}(\btheta^f,\btheta_{t}^{h})+ \kappa_1 \sum_{t=1}^T \|\balpha_t\|_2\nonumber\\
            & \quad + \kappa_2 \sum_{i,t=1}^{T} \balpha_{t,i} \hat{d}_{t,i}(\btheta^{f},\btheta_{t,i}^d),\label{alpha_loss}\\
        \mathrm{s.t.} \quad & \|\balpha_t\|_1 = 1, ~~~~ \balpha_{t,i}\geq 0 ~~~ \forall t,i,\nonumber
\end{align}
where $\kappa_1$ and $\kappa_2$ are hyper-parameters and $\hat{d}_{\cdot}$ is the estimated distribution ``distance''. This distribution ``distance'' may have different forms with according to the similarity metric used:
\begin{description}
    \item[$\calH$ divergence:] According to \cite{pentina2017multi,ben2010theory,ganin2016domain}, the distribution ``distance'' is proportional to the accuracy of the discriminator $\btheta^d_{\cdot}$, i.e., we applied $g_{t,i}(\x)$ to predict $\x$ coming from distribution $t$ or $i$. The prediction accuracy reflects the difficulty to distinguish two distributions. Hence, we set $\hat{d}_{t,i}$ as the accuracy of the discriminator $g_{t,i}(\x)$;
    \item[Wasserstein-1 distance:] According to \cite{arjovsky2017wasserstein}, the approximation $\hat{d}_{t,i} = - \hat{E}_{t,i}$ is used.
\end{description}

We also assume $\hat{d}_{t,t}=0$ since the discriminator cannot distinct from two identical distributions. Moreover, the expected loss $\balpha_t \lambda_{t,\cdot}$ is omitted since we assume that $\lambda_{t,\cdot}$ is much smaller than the empirical term. Then, we only use the empirical parts to re-estimate the relationship coefficient.

As it is mentioned in the theoretical part, the $L_2$ norm regularization aims at preventing all of the relation coefficient from being concentrated on the current task $\balpha_{t,t}$. The theoretical bound proposed an elegant interpretation for training AMTNN, which is shown in Algorithm \ref{amtnn_algo}.

\subsection{Training algorithm}
The general framework of the neural network is shown in Fig.~\ref{fig:deep_MTL}. We propose a complete iteration step on how to update the neural network parameters and relation coefficients in Algorithm \ref{amtnn_algo}. When updating the feature extraction parameter $\btheta^f$, we applied \emph{gradient reversal} \cite{ganin2016domain} in the training procedure. We also add the \emph{gradient penalty} \cite{gulrajani2017improved} to improve the Lipschitz property when training with the adversarial loss based on Wasserstein distance.

\section{Experiments}
We evaluate the modified AMTNN method on two benchmarks, that is the digits datasets and the Amazon sentiment dataset. We also consider the following approaches, as baselines to make comparisons: 
\begin{itemize}
    \item MTL\_uni: the vanilla MTL framework where $\frac{1}{T}\sum_{t=1}^T \hat{R}_t(\btheta^f,\btheta^h_t)$ is minimized;
    \item MTL\_weighted: minimizing $\frac{1}{T}\sum_{t=1}^T \hat{R}_{\balpha_t}(\btheta^f,\btheta^h_t)$, computation of $\balpha_t$ depending on $\hat{R}_{t,i}$, similarly to \cite{murugesan2016adaptive};
    \item MTL\_disH and MTL\_disW: we apply the same type of loss function but with two different adversarial losses ($\calH$ divergence and Wasserstein distance) and a general neural network without a special part for the NLP \cite{liu2017adversarial};
    \item AMTNN\_H and AMTNN\_W: proposed approaches with two different adversarial losses, $\calH$ divergence and Wasserstein distance respectively.
\end{itemize}

\subsection{Digit recognition}
We first evaluate our algorithm on three benchmark datasets of digit recognition, which are datasets, MNIST, MNIST\_M, and SVHN. The MTL setting is to jointly allow a system to learn to recognize the digits from the three datasets, which can differ significantly. In order to show the effectiveness of MTL, only a small portion of the original dataset is used for training (i.e., 3K, 5K and 8K for each task). 

We use the LeNet-5 architecture and define the feature extractor $\btheta^{f}$ as the two convolutional layers of the network, followed by multiple blocks of two fully connected layers as label prediction parameter $\btheta^{h}_{\cdot}$ and discriminator parameter $\btheta^{d}_{\cdot}$. Five repetitions are conducted for each approach, and the average test accuracy ($\%$) is reported in Table \ref{tab:digit}. We also show the estimated coefficient $\{\balpha_t\}_{t=1}^3$ of AMTNN\_H and AMTNN\_W, in Fig.~\ref{fig:relation}.

\begin{figure}[!tbp]
		\centering 
		\subfloat[AMTNN\_W]{\includegraphics[width=0.26\textwidth]{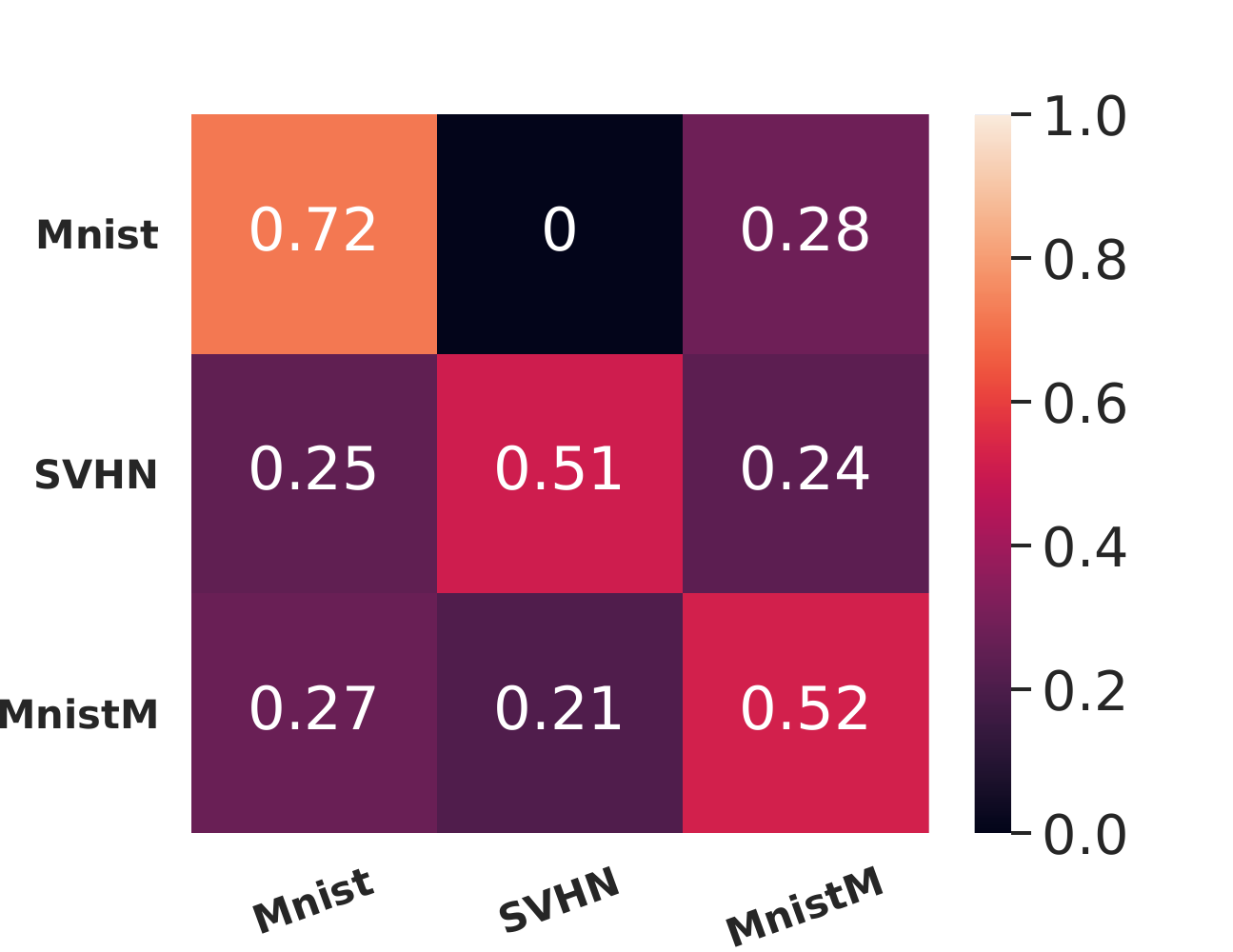}}
		\subfloat[AMTNN\_H]{\includegraphics[width=0.26\textwidth]{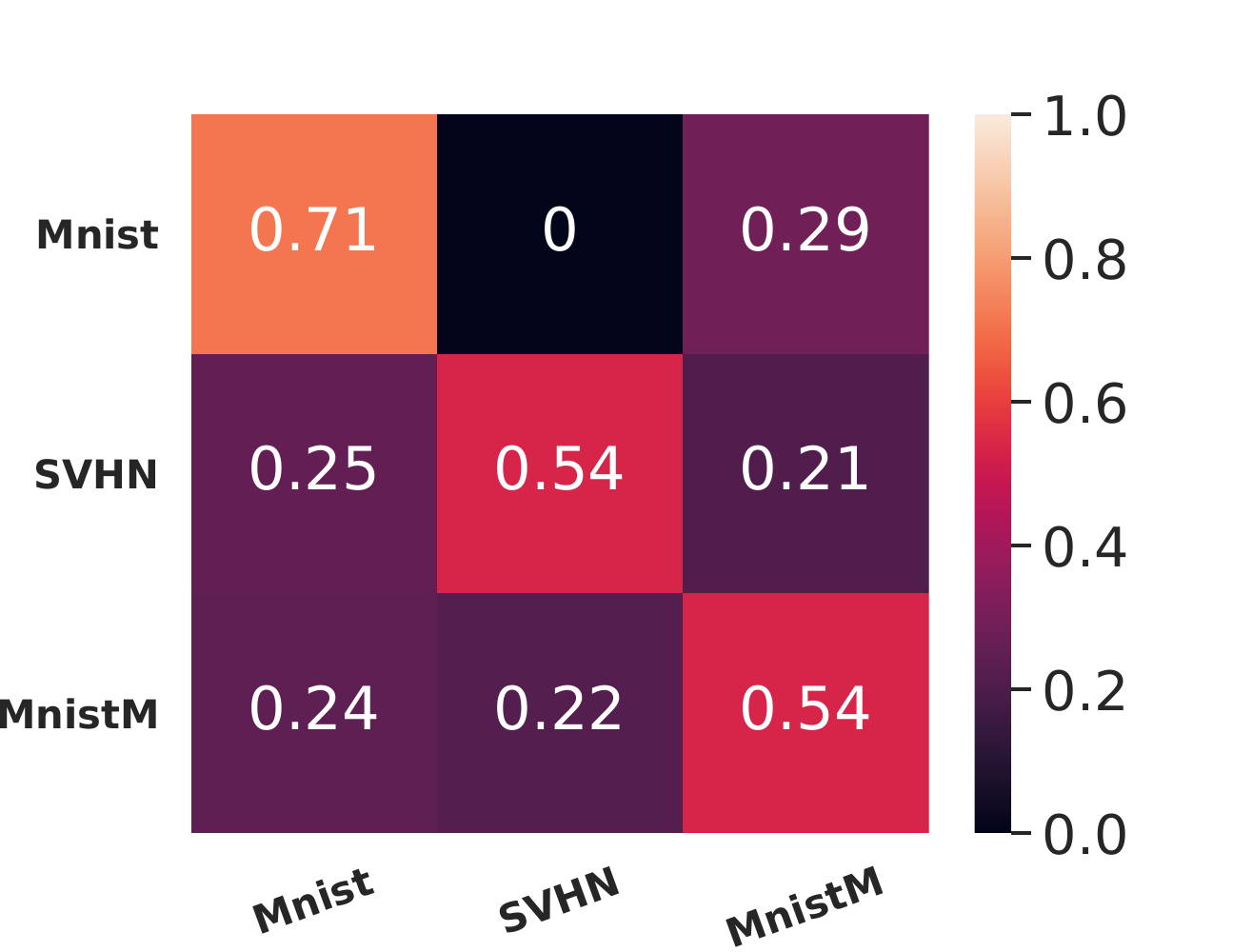}}
\caption{Estimated task relation coefficients matrix from the two proposed algorithms, with training set of 8K instances.}
\label{fig:relation}
\end{figure}

\begin{table*}[t]
    \centering
    \resizebox{1\textwidth}{!}{
        \begin{tabular}{l|ccc|c|ccc|c|ccc|c}
        \toprule
        & \multicolumn{4}{c|}{\textbf{3K}} & \multicolumn{4}{c|}{\textbf{5K}} & \multicolumn{4}{c}{\textbf{8K}} \\ 
        Approach & \textbf{MNIST} & \textbf{MNIST\_M}    & \textbf{SVHN}  & Average  & \textbf{MNIST} & \textbf{MNIST\_M}  & \textbf{SVHN} &  Average   & \textbf{MNIST}  & \textbf{MNIST\_M} & \textbf{SVHN}  & Average\\ 
        \midrule
        
        \multicolumn{1}{l|}{\textbf{MTL\_uni}} &  93.23    & 76.85     &57.20
        &  75.76    &  97.41 &   77.72   &  67.86    & 81.00   
        & 97.73  &  83.05   &  71.19    &   83.99  \\

        \multicolumn{1}{l|}{\textbf{MTL\_weighted}} &  89.09  & 73.69   
        &    68.63  &  77.13   & 91.43     &   74.07   &  73.81    & 79.77
        &  92.01    &    76.69  &    73.77  &   80.82  \\
        
        \multicolumn{1}{l|}{\textbf{MTL\_disH}} & 89.91     &   \textbf{81.13}   & 70.31
        & 80.45  &    91.92  & 82.68     &    73.27  &  82.62  &    92.96  &   \textbf{85.04}   &    78.50  &   85.50  \\
        \multicolumn{1}{l|}{\textbf{MTL\_disW}} & 96.77     & 80.38
        &   68.40 & 81.85   &  95.47   &   \textbf{83.48}   &   72.66
        &  83.87  
        &98.09     &   84.13 &    74.37  &    85.53 \\ \hline
        
        \multicolumn{1}{l|}{\textbf{AMTNN\_H}} & \textbf{97.47}     &  77.87   & 71.26
        & 82.20     &    \textbf{97.94}  &   76.28   &   76.06 &  83.43   &     \textbf{98.28}  &  82.75    & 76.63     &   85.89  \\ 
        \multicolumn{1}{l|}{\textbf{AMTNN\_W}} &   97.20   &  80.70    &  \textbf{76.93}
        &\textbf{84.95}     &  97.67    &    82.50  & \textbf{76.36}     &  \textbf{85.51}
        &   98.01   &   82.53   &  \textbf{79.97}    &  \textbf{86.84}   \\
        \bottomrule
\end{tabular}
}
    \caption{Average test accuracy (in $\%$) of MTL algorithms on the digits datasets.}
    \label{tab:digit}
\end{table*}

\begin{table*}[t]
    \centering
    \resizebox{0.95\textwidth}{!}{
        \begin{tabular}{l|cccc|c|cccc|c}
        \toprule
        & \multicolumn{5}{c|}{\textbf{1000 examples}} & \multicolumn{5}{c}{\textbf{1600 examples}}  \\ 
        Approach & \textbf{Book} & \textbf{DVDs}    & \textbf{Kitchen}  &  \textbf{Elec} & Average  & \textbf{Book} & \textbf{DVDs}    & \textbf{Kitchen}  &  \textbf{Elec} & Average   \\ 
        \midrule
        
        \multicolumn{1}{l|}{\textbf{MTL\_uni}} & 81.31     &  78.44    & 87.07
        &  84.57   &  82.85   &   81.35   &  80.14    & 86.54   
        & 87.50  &  83.88      \\

        \multicolumn{1}{l|}{\textbf{MTL\_weighted}} &  81.88  & 79.02    
        &  86.91    &   85.31  &    83.28  &   80.72   &  81.20    & 87.60
        &  88.12    &    84.41  \\
        
        \multicolumn{1}{l|}{\textbf{MTL\_disH}} &  81.23    & 78.12     &  87.34
        &  84.82  & 82.88     & \textbf{81.92}     &    79.86  &  87.79  &    87.31  &   84.22   \\
        
        \multicolumn{1}{l|}{\textbf{MTL\_disW}} & 81.13    & 78.38     &  87.11
        &  84.82   & 82.86   &   81.88   &   79.81
        &  87.07   & 87.69     &   84.11  \\ \hline
        
        \multicolumn{1}{l|}{\textbf{AMTNN\_H}} &   \textbf{82.36}   &  79.24    &  \textbf{87.42} 
        &   85.53  &    \textbf{83.64}  &   80.82   &   \textbf{81.54} &  \textbf{88.27}   &     88.17 &  \textbf{84.70}      \\ 
        
        \multicolumn{1}{l|}{\textbf{AMTNN\_W}} &  81.68    &  \textbf{79.38}     &  87.27
        &  \textbf{85.66}   & 83.50   &    81.20  & 80.38     &  87.69
        &   \textbf{88.46}   &   84.44    \\
        \bottomrule

\end{tabular}
}
    \caption{Average test accuracy (in $\%$) of MTL algorithms in the sentiment dataset.}
    \label{tab:amazon}
\end{table*}

\paragraph{Discussion} Reported results show that the proposed approaches outperform all of the baselines in the task average and also in most single tasks. 
Particularly for the AMTNN\_W, it outperforms the baselines with $1.0\% \sim 2.9\%$ in the test accuracy. The reason can be that the Wasserstein-1 distance is more efficient for measuring the high dimensional distribution, which has been verified theoretically \cite{redko2017theoretical}.  Moreover, the $\calH$ divergence-based approach (AMTNN\_H) outperforms the baselines without significant increment ($< 0.3\%$). The reason may be that the VC-dimension with $\calH$ divergence is not a good metric for measuring a high dimensional complex dataset, coherently with \cite{li2018extracting}.

As for the coefficients $\balpha_t$, the proposed algorithm appears robust at estimating these task relationships, with almost identical values under different similarity metrics. Moreover, in contrast to the previous approaches, we obtain a non-symmetric matrix with a better interpretability. For instance, when learning for the MNIST dataset, only information from MNIST\_M is used, which is reasonable since these two tasks have the same digits configurations with different background, while SVHN is different in most ways (i.e., digits taken from street view house numbers). However, when learning MNIST\_M, information from SVHN is beneficial because it provides some information on the background, which is absent from MNIST but similar to MNIST\_M. Therefore, the information of both tasks are involved in training for MNIST\_M.

\begin{figure}[tb]
		\centering 
		\includegraphics[width=0.475\textwidth]{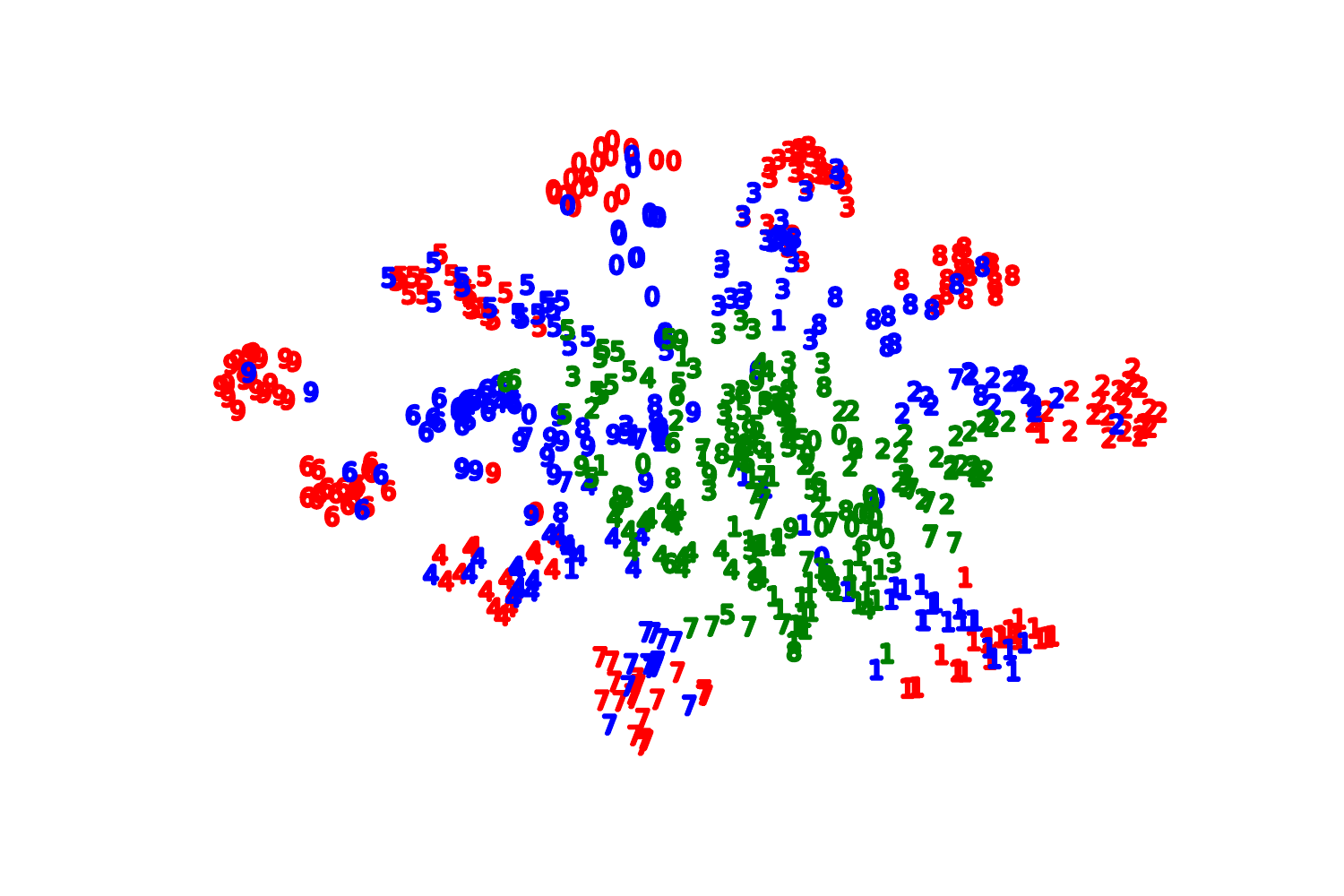}
\caption{t-SNE in the feature space of task MNIST in AMTNN\_W for 8K. Red: MNIST dataset; blue: MNIST\_M dataset; green: SVHN data set.}
\label{fig:tsne}
\end{figure}

In order to show the role of the weighted sum, we use t-SNE to visualise in Fig.~\ref{fig:tsne} the embedded space of the MNIST task from the training data. Information from SVHN is not relevant for learning MNIST as $\balpha_{1,2}=0$ (see Fig.~\ref{fig:relation}), such that SVHN data is arbitrarily distributed in the embedded space without influence on the final result. At the same time, information from MNIST\_M is used for training on the MNIST task ($\balpha_{1,3} = 0.28$), which can be seen by a slight overlap in the embedded space. From that perspective, the role of weighted loss, which helps us to achieve some reasonable modifications of the decision boundary, is trained by the relevant and current tasks jointly. For a small scale task (typically the MTL scenario), during the test procedure, the agent predicts the labels by borrowing its neighbors (relevant tasks) information. This is coherent with the Probabilistic Lipschitzness condition \cite{urner2013probabilistic}.

\subsection{Sentiment analysis}
We also evaluate the proposed algorithm on \emph{Amazon reviews} datasets. We extract reviews from four product categories: Books, DVD, Electronics and Kitchen appliances. Reviews datasets are pre-processed with the same strategy proposed by \cite{ganin2016domain}:  10K dimensional input features of uni-gram/bi-gram occurrences and binary output labels $\{0,1\}$, Label 0 is given if the product is ranked less than 3 stars, otherwise label 1 is given for products above 3 stars. Results are reported for two sizes of labelled training sets, that is $1000$ and $1600$ examples in each product category.

The output of the first fully connected layers as feature extractor parameters $\btheta^f$ and several sets of two fully-connected layers are given as discriminator $\btheta^d_{\cdot}$ and label predictor $\btheta^h_{\cdot}$, with test accuracy ($\%$) reported in Table \ref{tab:amazon} as an average over 5 repetitions.  

\paragraph{Discussions}
We found the proposed approaches outperform all of the baselines in the task average and also in the most tasks. Meanwhile, we observed that the role of adversarial loss (MTL\_disH, MTL\_disW, AMTNN\_H and AMTNN\_W) is not that significant (gains $<0.25\%$), compared to the results on the digits datasets. The possible reason is that we applied the algorithm on the pre-processed feature instead of the original feature, making the discriminator $\btheta^d_{\cdot}$ less powerful in the feature adaptation. By the contrary, adding the weighted loss can improve performance by $0.4\% \sim 0.9\%$, enhancing the importance of the role of the explicit similarity, which is also coherent with \cite{murugesan2016adaptive}. 

\section{Conclusion}
In this paper, we propose a principle approach for using the task similarity information in the MTL framework. We first derive an upper bound of the generalization error in the MTL. Then, according to the theoretical results, we design a new training algorithm on the Adversarial Multi-task Neural Network (AMTNN). Finally, the empirical results on the benchmarks are showing that the proposed algorithm outperforms the baselines, reaffirming the benefits of theoretical insight in the algorithm design. In the future, we want to extend to a more general transfer learning scenario such as the different outputs space. 

\section*{Acknowledgments}

This work was made possible with funding from NSERC-Canada, Mitacs, Prompt-Qu\'ebec, E Machine Learning and Thales Canada. We thank Annette Schwerdtfeger and Fan Zhou for proofreading this manuscript.

\newpage
\onecolumn
\section*{Proof of Theorem 1}
\begin{theorem*}
Let $\calH$ be a hypothesis family with VC dimension $d$. If we have $T$ tasks generated by the underlying distribution and labelling function $\{(\D_1,f_1),\dots, (\D_T,f_T)\}$ with observation numbers $m_1,\dots,m_T$. If we adopt $\mathcal{H}$ divergence as a similarity metric, then for any fixed simplex $\balpha_{t}\in\mathbb{R}^{T}_{+}$, and for $\delta\in(0,1)$, with probability at least $1-\delta$, for $h_1,\dots,h_T \in\mathcal{H}$, we have:
\begin{equation*}
\frac{1}{T}\sum_{t=1}^T R_t(h_t)  \leq  \underbrace{\frac{1}{T}\sum_{t=1}^T \hat{R}_{\balpha_t}(h_t)}_\text{Weighted empirical loss}+ \underbrace{C_1  \sum_{t=1}^T \Big( \sqrt{ \sum_{i=1}^T \frac{\balpha^2_{t,i}}{\beta_i}} \Big)}_\text{Coefficient regularization} + \underbrace{\frac{1}{T}\sum_{t=1}^T \sum_{i=1}^T \balpha_{t,i} \hat{d}_{\calH}(\D_t,\D_i)}_\text{Empirical distribution distance} + \underbrace{C_2 + \frac{1}{T}\sum_{t=1}^T\sum_{i=1}^T \balpha_{t,i}\lambda_{t,i}}_\text{Complexity term and optimal expected loss}
\end{equation*}
Where $\beta_i = \frac{m_i}{m}$, $C_1 = 2 \sqrt{ \frac{2(d\log(\frac{2em}{d})+\log(\frac{16T}{\delta}))}{m}}$ and $C_2 =2\min_{i,j} \sqrt{\frac{2d\log(2m_{i,j})+ \log(32T/\delta)}{m_{i,j}}}$ with $m_{i,j} = \min\{ m_i,m_j\}$, and $\lambda_{i,j} = \inf_{h\in\calH} \{R_i(h) + R_j(h)\}$ the \emph{joint expected minimal error} w.r.t. $\calH$.
\end{theorem*}

\section*{Theoretical tools}
In this section, we will list the theoretical tools, which will be applied multiple times in the later proof. 

\subsection*{Transfer bounds}
In this section, we will analyze the relations of the \emph{expected risk} on the distributions.
\begin{lemma}
\cite{ben2010theory} Let $\calH$ be the hypothesis space with VC dimension $d$. For two tasks with marginal w.r.t $x$ generation distribution $\D_i$ and $\D_j$. For every $h\in\calH$:
\begin{equation}
R_j(h) \leq R_i(h) + d_{\calH}(D_i,D_j) + \lambda_{i,j}
\label{hdiv_diss}
\end{equation}
Where $\lambda_{i,j} = \inf_{h\in\calH} \{R_i(h)+R_j(h)\}$
\end{lemma} 

\subsection*{Concentration bounds between empirical and expected divergence}
\begin{lemma}
Let $\calH$ be the hypothesis space on $\calX$ with VC dimension $d$. If $S_i$ and $S_j$ are the i.i.d samples with size $m_i$ and $m_j$, respectively. We also define the empirical divergence $\hat{d}_{\calH}(\D_i,\D_j)$ w.r.t. $S_1$ and $S_2$, then for any $\delta \in (0,1)$, with probability at least $1-\delta$:
\begin{equation}
 d_{\calH}(\D_i,\D_j) \leq \hat{d}_{\calH}(\D_i,\D_j) + 2\sqrt{\frac{2d\log(2m_{ij})+ \log(\frac{2}{\delta})}{m_{ij}}}
\label{Hdiv_concen}
\end{equation} 
Where $m_{ij} = \min \{m_i, m_j\}$.
\end{lemma}

The bound is slight different from \cite{ben2010theory}, because the original paper supposed the equal number of observations between two distributions. However, the proof is also a simple plugging in the conclusion of \cite{kifer2004detecting} . 

\begin{proof}
From the Theorem 3.4 of \cite{kifer2004detecting}, we have:
\begin{equation}
    P[|d_{\calH}(\D_i,\D_j) - \hat{d}_{\calH}(\D_i,\D_j)|\geq \epsilon ]  \leq (2m_i)^d e^{-m_i \epsilon^2/16} + (2m_j)^d e^{-m_j \epsilon^2/16}
\label{vldb2004}
\end{equation}
We consider the function $f(x) = (2x)^d \exp(-x\epsilon^2/16)$, then we compute the gradient w.r.t $x$, we have $f^{\prime}(x)  = (2x)^{d-1}\exp(-x\epsilon^2/16) (2d-2x\epsilon^2/16) < 0$.
When $\epsilon^2 \geq 16 d/m_{i,j}$, we have $f(m_j), f(m_j)$ can both be upper bounded by $f(m_{i,j})$.

Hence we can verify the R.H.S. in equation (\ref{vldb2004}) can be upper bounded by $ 2(2m_{i,j})^d e^{-m_{i,j} \epsilon^2/16}$. Then we set this value as $\delta$, we have $\epsilon^2 \geq 16 \frac{\log(2/\delta) + d\log(2m_{i,j})}{m_{i,j}}$. 

Under this condition, we have the conclusion showed in the Lemma. Moreover, the divergence is defined for the $\calH$ hypothesis set, then the VC dimension is $2d$ for such a hypothesis set.  
\end{proof}

\subsection*{Concentration bounds between empirical and expected risk}
Another useful inequality is to bound the difference between the empirical and the expected error in the weighted loss.
\begin{lemma}
For each task index $j= \{1,\dots,T\}$, let $S_j$ be a labeled sample of size $\beta_j n$ generated from distribution $\D_j$ and labeled according to the function $f_j$. For the any fixed $\balpha$, and any binary classifiers $h\in\mathcal{H}: \calX \to \{-1,1\}$ with VC dimension $d$. With probability greater than $1-\delta$, we have:
\begin{equation}
 R_{\balpha_t}(h) \leq \hat{R}_{\balpha_t}(h) +  2\sqrt{ \sum_{j=1}^T \frac{\balpha^2_{t,j}}{\beta_j}} \sqrt{ \frac{2(d\log(\frac{2en}{d})+\log(\frac{8}{\delta}))}{n}}
 \label{class_bound}
\end{equation}
\end{lemma}

\begin{proof}
The proof is analogue to the proof in uniform convergence bound. We first apply the symmetrization trick by generating ghost samples. For the notation simplification, we define $\hat{R}^{\prime}_{\balpha_t}$ is the empirical risk induced by $Z^{\prime}_1,Z^{\prime}_2,\dots$ by sampling from the same distribution (but we never know it, so we called \emph{ghost sample}). 

From the symmetrization lemma, we have, for $\epsilon \geq \sqrt{2/n}$:
\begin{equation}
        \Proba \big(\sup_{h\in\mathcal{H}} |R_{\balpha_t}(h)- \hat{R}_{\balpha_t}(h)| \geq \epsilon \big) \leq \\
       2 \Proba \big(\sup_{h\in\mathcal{H}} |\hat{R}_{\balpha_t}(h) - \hat{R}^{\prime}_{\balpha_t}(h)| \geq \frac{\epsilon}{2} \big)
\end{equation}

Then we prove the modified the VC-bound, defining $V = \mathcal{H}_{Z_1,\dots,Z_n,Z^{\prime}_1,\dots,Z_{n}}$. For any $v\in V$, we can write $\hat{R}_{\balpha_t}(h) - \hat{R}^{\prime}_{\balpha_t}(h) = \frac{1}{n} \sum_{j=1}^n v_j - \sum_{j=n+1}^{2n} v_j$

\begin{equation}
        \Proba \big( \sup_{h\in\mathcal{H}} |\hat{R}_{\balpha_t}(h) - \hat{R}^{\prime}_{\balpha_t}(h)| \geq \frac{\epsilon}{2} \big) 
        \leq 2 \Proba \big( \max_{v\in V} |\frac{1}{n} \sum_{j=1}^n v_j - \sum_{j=n+1}^{2n} v_j | \geq \frac{\epsilon}{2} \big)  
\end{equation}
By union bound, we have:
$$\leq 2 \Pi(2n) \Proba \big( |\frac{1}{n} \sum_{j=1}^n v_j - \sum_{j=n+1}^{2n} v_j | \geq \frac{\epsilon}{2} \big) $$
By introducing the Rademacher variable $\sigma_j$, we have:
$$\leq 4  \Pi(2n) \Proba \big( |\frac{1}{n} \sum_{j=1}^n \sigma_j v_j | \geq \frac{\epsilon}{4} \big)$$
According the Hoeffding's inequality, we have:
$$ \leq 8 \Pi(2n) \exp(-\frac{n\epsilon^2}{8 \sum_{j=1}^T (\frac{\balpha_{t,j}}{\beta_j})^2 }) $$
Applying Sauer's lemma we have at probability at least $1-\delta$:
$$ R_{\balpha_t}(h) \leq \hat{R}_{\balpha_t}(h) +  2\sqrt{ \sum_{j=1}^T \frac{\balpha^2_{t,j}}{\beta_j}} \sqrt{ \frac{2(d\log(\frac{2en}{d})+\log(\frac{8}{\delta}))}{n}} $$
This bound can also be proved with Mcdiarmid inequality and Rademacher complexity with sightly loose than our demonstration.
\end{proof}

\subsection*{Three steps proof}
In this section, we try to make connections between the similarity measuring and the expected risk. 
For a pair of distribution $(\D_i,\D_j)$, we define $h^{\star}_{i,j} \in \mathrm{argmin}_{h\in\calH} \{R_i(h) + R_j(h)\}$ the \emph{joint expected minimal error} for the hypothesis class $\calH$. 

\paragraph{Step 1:} For one task $t$ we have:
\begin{equation}
|R_{\balpha_t}(h) - R_t(h)| = |\sum_{i=1}^T \balpha_{t,i} R_i(h) - R_t(h)|\leq \sum_{i=1}^T \balpha_{t,i} |R_i(h)-R_t(h)|
\end{equation}
According to the triangle inequality of the loss function, we have
\begin{equation}
    \leq  \sum_{i=1}^T \balpha_{t,i} \Big(|R_i(h) - R_i(h,h_{i,t}^{\star})| + |R_i(h,h_{i,t}^{\star})- R_t(h,h_{i,t}^{\star})| + |R_t(h) - R_t(h,h_{i,t}^{\star})| \Big)
\label{mul_01}
\end{equation}
According to the triangle inequality and definition of the distribution discrepancy, we have:
$$|R_i(h)-R_i(h,h^{\star}_{i,t})| \leq R_i(h^{\star}_{i,t})$$
$$|R_i(h,h_{i,t}^{\star})- R_t(h,h_{i,t}^{\star})|\leq d_{\calH}(D_i,D_j) $$
$$|R_t(h) - R_t(h,h_{i,t}^{\star})| \leq  R_t(h_{i,t}^{\star})$$
Plugging in (\ref{mul_01}), we have:
\begin{equation}
    \begin{split}
        \leq & \sum_{i=1}^T \balpha_{t,i}(R_i(h^{\star}_{i,t}) + R_t(h_{i,t}^{\star}) + d_{\calH\Delta\calH}(D_i,D_j))\\ 
        & = \sum_{i=1}^T  \balpha_{t,i}(\lambda_{t,i} + d_{\calH\Delta\calH}(D_i,D_j))
    \end{split}
\end{equation}

\begin{equation}
= \sum_{i=1}^T \balpha_{t,i}\lambda_{t,i} + \sum_{i=1}^T \balpha_{t,i}d_{\calH\Delta\calH}(D_i,D_j)
\label{mul_02}
\end{equation}

Finally for $t= 1,\dots,T$ tasks, the expected risk can be upper bounded by:
\begin{equation}
        \frac{1}{T}\sum_{t=1}^T R_t(h_t) \leq \frac{1}{T}\sum_{t=1}^T R_{\balpha_t}(h_t) + 
       \frac{1}{T}\sum_{t=1}^T \sum_{i=1}^T \balpha_{t,i} d_{\calH\Delta\calH}(D_t,D_i) + \frac{1}{T}\sum_{t=1}^T\sum_{i=1}^T \balpha_{t,i}\lambda_{t,i}
\label{pair_delta}
\end{equation}
The next step is to find the high probability bound to measure the expected and empirical terms. 

\paragraph{Step 2:} With probability at least $1-\delta/2$, the expected discrepancy can be upper bounded by:
\begin{equation}
     \frac{1}{T}\sum_{t=1}^T \sum_{i=1}^T \balpha_{t,i} d_{\calH\Delta\calH}(\D_t,\D_i) \leq 
\frac{1}{T}\sum_{t=1}^T \balpha_{t,i} \hat{d}_{\calH\Delta\calH}(S_t,S_i)+ 2 \sqrt{\frac{2d\log(2m_{\star})+ \log(32T/\delta)}{m_{\star}}}
\label{div_diff}
\end{equation} 
Where $m_{\star} = \mathrm{argmin}_{m_{i,j}}~~\sqrt{\frac{2d\log(2m_{i,j})+ \log(32T/\delta)}{m_{i,j}}}$

\begin{proof}
For task $t$, from Corollary \ref{Hdiv_concen}, we have with probability less than $\delta^{\prime}$, we have:
$$ d_{\calH\Delta\calH}(\D_t,\D_i) \geq \hat{d}_{\calH\Delta\calH}(S_t,S_i) + 2\sqrt{\frac{2d\log(2m_{ti})+ \log(2/\delta^{\prime})}{m_{ti}}} $$
Then we have
\begin{equation}
    \balpha_{t,i} d_{\calH\Delta\calH}(\D_t,\D_i) \geq \balpha_{t,i} \hat{d}_{\calH\Delta\calH}(S_t,S_i) + 2 \balpha_{t,i} \sqrt{\frac{2d\log(2m_{ti})+ \log(2/\delta^{\prime})}{m_{ti}}}
\end{equation}

Then we set $\delta^{\prime} = \delta/(4T)$, then apply the union bound, we know $\exists h$, such that
\begin{equation}
   \sum_{i=1}^T \balpha_{t,i} d_{\calH\Delta\calH}(\D_t,\D_i) \geq \sum_{i=1}^T \balpha_{t,i} \hat{d}_{\calH\Delta\calH}(S_t,S_i) + \sum_{i=1}^T 2 \balpha_{t,i} \sqrt{\frac{2d\log(2m_{t\star})+ \log(8T/\delta)}{m_{t\star}}}
\end{equation}
$$ = \sum_{i=1}^T \balpha_{t,i} \hat{d}_{\calH\Delta\calH}(S_t,S_i) + 2 \sqrt{\frac{2d\log(2m_{t\star})+ \log(8T/\delta)}{m_{t\star}}} $$
Where $m_{t\star} = \mathrm{argmin}_{m_{t,i}}~~\sqrt{\frac{2d\log(2m_{ti})+ \log(8T/\delta^{\prime})}{m_{ti}}}$, is the $m_{t,i}$ which has the smallest complexity term.

Then we again apply the union bound over $t$, finally there exists a hypothesis $h$ with probability smaller than $\delta/2$, holding the following bound:
\begin{equation}
     \frac{1}{T}\sum_{t=1}^T \sum_{i=1}^T \balpha_{t,i} d_{\calH\Delta\calH}(\D_t,\D_i) \geq 
\frac{1}{T}\sum_{t=1}^T \balpha_{t,i} \hat{d}_{\calH\Delta\calH}(S_t,S_i) + 2 \sqrt{\frac{2d\log(2m_{\star})+ \log(32T/\delta)}{m_{\star}}}
\end{equation}

Where $m_{\star} = \mathrm{argmin}_{m_{i,j}}~~\sqrt{\frac{2d\log(2m_{i,j})+ \log(32T/\delta)}{m_{i,j}}}$ is the $m_{i,j}$ which has the smallest complexity term. Finally we have (\ref{div_diff}) under high probability $1-\delta/2$.
\end{proof}
\paragraph{Step 3:} Applying the union bound and we have the high probability at least $1-\frac{\delta}{2}$, we have:

\begin{equation}
    \frac{1}{T}\sum_{t=1}^T R_{\balpha_t}(h_t) \leq \frac{1}{T}\sum_{t=1}^T \hat{R}_{\balpha_t}(h_t) 
    + 2 \sqrt{ \frac{2(d\log(\frac{2em}{d})+\log(\frac{16T}{\delta}))}{m}} \sum_{t=1}^T \Big( \sqrt{ \sum_{j=1}^N \frac{\balpha^2_{t,j}}{\beta_j}} \Big)
\end{equation}

\paragraph{Step 4:} Combining previous conclusions, we have at high probability at least $1-\delta$:
\begin{equation*}
     \frac{1}{T}\sum_{t=1}^T R_t(h_t)  \leq \frac{1}{T}\sum_{t=1}^T \hat{R}_{\balpha_t}(h_t) + C_1  \sum_{t=1}^T \Big( \sqrt{ \sum_{i=1}^T \frac{\balpha^2_{t,i}}{\beta_i}} \Big) 
     + \frac{1}{T}\sum_{t=1}^T \sum_{i=1}^T \balpha_{t,i} \hat{d}_{\calH\Delta\calH}(S_t,S_i)+ C_2 + \frac{1}{T}\sum_{t=1}^T\sum_{i=1}^T \balpha_{t,i}\lambda_{t,i} 
\end{equation*}
Where $C_1 = 2 \sqrt{ \frac{2(d\log(\frac{2em}{d})+\log(\frac{16T}{\delta}))}{m}} $, $C_2 =2 \sqrt{\frac{2d\log(2m_{\star})+ \log(32T/\delta)}{m_{\star}}}$

\section*{Proof of Theorem 2}
\begin{theorem*}
Let $\calH$ be a hypothesis family from $\calX$ to $[0,1]$, with pseudo-dimension $d$ and each member $h\in\calH$ is $K$ Lipschtiz. If we have $T$ tasks generated by the underlying distribution and labelling function $\{(\D_1,f_1),\dots, (\D_T,f_T)\}$ with observation numbers $m_1,\dots,m_T$. If we adopt  Wasserstein-1 \footnote{This bound can be extended to any Wasserstein $p>1$ distance with restricting the hypothesis satisfies $K$ H\"{o}lder condition.} distance as a similarity metric with cost function $c(\x,\y)= \|\x-\y\|_2$, then for any fixed simplex $\balpha_{t}\in\mathbb{R}^{T}_{+}$, and for $\delta\in(0,1)$, with probability at least $1-\delta$, for $h_1,\dots,h_T \in\mathcal{H}$, we have:
\begin{equation*}
 \frac{1}{T}\sum_{t=1}^T R_t(h_t) \leq  \underbrace{\frac{1}{T}\sum_{t=1}^T \hat{R}_{\balpha_t}(h_t)}_\text{Weighted empirical loss} +  \underbrace{C_1  \sum_{t=1}^T \Big( \sqrt{ \sum_{j=1}^T \frac{\balpha^2_{t,j}}{\beta_j}} \Big)}_\text{Coefficient regularization} 
 +  \underbrace{\frac{2K}{T}\sum_{t=1}^T \sum_{i=1}^T \balpha_{t,i} W_1 (\hat{D}_t,\hat{D}_i)}_\text{Empirical distribution distance} +\underbrace{C_2 +\frac{1}{T}\sum_{t=1}^T\sum_{i=1}^T \balpha_{t,i}\lambda_{t,i}}_\text{Complexity term and optimal expected loss}
\label{wbound}
\end{equation*}
Where $\beta_i = \frac{m_i}{m}$, $C_1 = 2 \sqrt{ \frac{2(d\log(\frac{2em}{d})+\log(\frac{16T}{\delta}))}{m}}$, $C_2 = \frac{2K}{T} \sum_{t=1}^T \sum_{i=1}^T \gamma_{t,i}$
with $\gamma_{t,i} = \mu_t m_t^{-1/s} + \mu_i m_i^{-1/s} +  \sqrt{\log(\frac{2T}{\delta})}(\sqrt{\frac{1}{m_t}} + \sqrt{\frac{1}{m_i}})$ and $s$ and $\mu_{\cdot}$ are some specified constants.  
\end{theorem*}

The proof is analogue to the proof in Theorem 1 with some different assumptions.
\subsection*{Transfer bounds}
The proof extends the work of \cite{redko2017theoretical} where the hypothesis is only restricted in the unit ball of RKHS. We extend this result to any hypothesis with Lipschitz function.
\begin{lemma}
Let $\D_i$ and $\D_j$ be two probability measures on $\calX$. Assume that:
\begin{enumerate}
	\item Cost function is the Euclidean distance, with the form $c(\x,\y) = \|\x-\y\|$
	\item The hypothesis set $\mathcal{H}$ satisfies $K$-Lipschitz continuous: $\forall h\in\mathcal{H}$, $h$ is $K$-Lipschtiz continuous. 
\end{enumerate}
Then we have the following result:
\begin{equation}
R_j (h,h^{\prime}) \leq R_i(h,h^{\prime}) + 2K W_1(\D_i,\D_j)
\label{wass_dis}
\end{equation}
for any hypothesis $h,h^{\prime}\in\mathcal{H}$
\end{lemma}

\begin{proof}
According to the definition of the expected risk, we have:
$$R_j(h,h^{\prime}) = R_j(h,h^{\prime}) +  R_i(h,h^{\prime}) - R_i(h,h^{\prime})$$
$$ \leq R_i(h,h^{\prime}) + |R_j(h,h^{\prime})-  R_i(h,h^{\prime})|$$
$$ \leq R_i(h,h^{\prime}) + |\E_{y\sim\D_j}|h(y)-h^{\prime}(y)| - \E_{x\sim\D_i}|h(x)-h^{\prime}(x)| |$$
By defining $\phi(x) =|h(x)-h^{\prime}(x)|$, we have:
$$ = R_i(h,h^{\prime}) + |\int_{\calX}\phi d (D_j-D_i)|$$
$$ = R_i(h,h^{\prime}) + |\int_{\calX\times\calX}\phi(x)-\phi(y) d \gamma(x,y) |$$
For \textbf{any} joint measure $\gamma(x,y)$, we have:
$$ \leq R_i(h,h^{\prime}) + \int_{\calX\times\calX}|\phi(x)-\phi(y)| d \gamma(x,y) $$
Thus it will also satisfy the minimal w.r.t $\gamma(x,y)$:
$$ \leq R_i(h,h^{\prime}) + \min_{\gamma(x,y)\in\Pi(\D_i,\D_j)}\int_{\calX\times\calX}|\phi(x)-\phi(y)| d \gamma(x,y)$$
We also have $|\phi(x) - \phi(y)|= ||h(x)-h^{\prime}(x)|-|h(y)-h^{\prime}(y)||\leq |h(x)-h^{\prime}(x)-h(y)+h^{\prime}(y)|\leq 2K \|x-y\|$, plugging in we have:
$$ R_j(h,h^{\prime}) \leq R_i(h,h^{\prime}) + 2K W_1(\D_i,\D_j) $$
\end{proof}

\textbf{Remark} If the function satisfies $(C,p)$-H\"{o}lder condition with $|h(x)-h(y)|\leq C \|x-y\|^p $, then the conclusion can be extended to any $p$-Wasserstein distance.

\subsection*{Concentration bounds between empirical and expected divergence}
There exists several concentration bounds such as \cite{bolley2007quantitative,weed2017sharp}, we adopt the conclusion from \cite{weed2017sharp} and apply to bound the empirical measures in Wasserstein distance.

\begin{lemma} \cite{weed2017sharp} [Definition 3,4] Given a measure $\mu$ on $X$, the $(\epsilon,\tau)$-covering number on a given set $S\subseteq X$ is:
$$\mathcal{N}_{\epsilon}(\mu,\tau) := \inf \{ \mathcal{N}_{\epsilon}(S): \mu(S)\geq 1-\tau \} $$
and the $(\epsilon,\mu)$-dimension is:
$$d_{\epsilon}(\mu,\tau) := \frac{\log \mathcal{N}_{\epsilon}(\mu,\epsilon)}{-\log\epsilon} $$
Then the upper Wasserstein dimensions can be defined as:
$$ d_p^{\star}(\mu) = \inf\{s\in(2p,+\infty): \lim\sup_{\epsilon \to 0} d_{\epsilon}(\mu,\epsilon^{-\frac{sp}{s-2p}})\leq s \}$$
\end{lemma}

\begin{lemma} \cite{weed2017sharp}[Theorem 1, Proposition 20] For $p\geq 1$ and $s \geq d_p^{\star}(\mu)$, there exists a positive constant $C$ with probability at least $1-\delta$, we have:
$$W_p^{p}(\mu,\hat{\mu}_n ) \leq  Cn^{-1/s} + \sqrt{\frac{1}{2n}\log(\frac{1}{\delta})} $$
\end{lemma}

\subsection*{Concentration bounds between empirical and expected risk}
In the regression problem, we suppose the hypothesis family $\mathcal{H}$ is a set of continuous mapping with \emph{pseudo-dimension} $d$. Then we can directly apply the conclusion of (\ref{class_bound}).
The procedure is analogue to the proof in $\mathcal{H}$ divergence but under different assumptions. \paragraph{Step 1:} For a pair of distribution $(\D_i,\D_j)$, for the task $t$ we have:
\begin{equation*}
|R_{\balpha_t}(h) - R_t(h)| = |\sum_{i=1}^T \balpha_{t,i} R_i(h) - R_t(h)|\leq \sum_{i=1}^T \balpha_{t,i} |R_i(h)-R_t(h)|
\end{equation*}

\begin{equation*}
     \leq \sum_{i=1}^T \balpha_{t,i} \Big(|R_i(h) - R_i(h,h_{i,t}^{\star})| + |R_i(h,h_{i,t}^{\star})- R_t(h,h_{i,t}^{\star})| + |R_t(h) - R_t(h,h_{i,t}^{\star})| \Big)
\end{equation*}
According to the triangle inequality and the previous lemma, we have:
$$|R_i(h)-R_i(h,h^{\star}_{i,t})| \leq R_i(h^{\star}_{i,t})$$
$$|R_i(h,h_{i,t}^{\star})- R_t(h,h_{i,t}^{\star})|\leq 2K W_1(\D_i,\D_j) $$
$$|R_t(h) - R_t(h,h_{i,t}^{\star})| \leq  R_t(h_{i,t}^{\star})$$
Plugging in, we have:
\begin{equation}
    \leq \sum_{i=1}^T \balpha_{t,i}(R_i(h^{\star}_{i,t}) + R_t(h_{i,t}^{\star})  + 2K W_1(D_i,D_j)) = \sum_{i=1}^T  \balpha_{t,i}(\lambda_{t,i} +  2K W_1(D_i,D_j))
\end{equation}
\begin{equation}
= \sum_{i=1}^T \balpha_{t,i}\lambda_{t,i} + 2K \sum_{i=1}^T \balpha_{t,i} W_1(D_i,D_j)
\end{equation}
Summing over the $t=1,\dots,T$:
\begin{equation}
     \frac{1}{T}\sum_{t=1}^T R_t(h_t) \leq \frac{1}{T}\sum_{t=1}^T R_{\balpha_t}(h_t) + \frac{2K}{T}\sum_{t=1}^T \sum_{i=1}^T \balpha_{t,i} W_1 (D_t,D_i) + \frac{1}{T}\sum_{t=1}^T\sum_{i=1}^T \balpha_{t,i}\lambda_{t,i}
\end{equation}

\paragraph{Step 2:} The next step is to bound the expected and the empirical Wasserstein distance. 
According to the triangle inequality of Wasserstein distance, we have:
\begin{equation}
    W_1(D_t,D_i) \leq  W_1(D_t,\hat{\D}_t) + W_1(\hat{\D}_t,\D_i)
   \leq W_1(D_t,\hat{\D}_t) + W_1(\hat{\D}_t,\hat{\D}_i) + W_1(\hat{\D}_i,\D_i) 
\end{equation}

According to the concentration lemma, we have with probability $1-\delta^{\prime}/2$:
\begin{equation}
         W_1(D_t,D_i) \leq  W_1(\hat{\D}_t,\hat{\D}_i) + C_t m_t^{-1/s} + C_i m_i^{-1/s} 
        + \sqrt{\frac{1}{2}\log(\frac{2}{\delta^{\prime}})}(\sqrt{\frac{1}{m_t}} + \sqrt{\frac{1}{m_i}}
\end{equation}

Then setting $\delta^{\prime} = \frac{\delta}{T^2}$ and applying union bound, we have the following with probability at least $1-\delta/2$:
\begin{equation}
     \frac{1}{T}\sum_{t=1}^T \sum_{i=1}^T \balpha_{t,i} W_1 (D_t,D_i) \leq \frac{1}{T}\sum_{t=1}^T \sum_{i=1}^T \balpha_{t,i} W_1 (\hat{D}_t,\hat{D}_i) + \frac{1}{T} \sum_{t=1}^T \sum_{i=1}^T \gamma_{t,i}
\end{equation}
Where $$\gamma_{t,i} =  C_t N_t^{-1/s} + C_i N_i^{-1/s} +  \sqrt{\log(\frac{2T}{\delta})}(\sqrt{\frac{1}{N_t}} + \sqrt{\frac{1}{N_i}})$$

\paragraph{Step 3:} 
Then the next step is to bound the empirical and expected error. Since here is the regression problem, we suppose the hypothesis family $\mathcal{H}$ is a set of continuous mapping with \emph{pseudo-dimension} $d$. Then we combine with the previous lemma, with probability at least $1-\delta$, the expected error can be upper bounded by:

\begin{equation}
\frac{1}{T}\sum_{t=1}^T R_t(h_t) \leq \frac{1}{T}\sum_{t=1}^T \hat{R}_{\balpha_t}(h_t) +  C_1  \sum_{t=1}^T \Big( \sqrt{ \sum_{j=1}^T \frac{\balpha^2_{t,j}}{\beta_j}} \Big) 
 +  \frac{2K}{T}\sum_{t=1}^T \sum_{i=1}^T \balpha_{t,i} W_1 (\hat{D}_t,\hat{D}_i) + C_2 +\frac{1}{T}\sum_{t=1}^T\sum_{i=1}^T \balpha_{t,i}\lambda_{t,i}
\label{wbound}
\end{equation}
Where $C_1 = 2 \sqrt{ \frac{2(d\log(\frac{2en}{d})+\log(\frac{16T}{\delta}))}{n}}$, $C_2 = \frac{2K}{T} \sum_{t=1}^T \sum_{i=1}^T \gamma_{t,i}$
with $\gamma_{t,i} = C_t m_t^{-1/s} + C_i m_i^{-1/s} +  \sqrt{\log(\frac{2T}{\delta})}(\sqrt{\frac{1}{m_t}} + \sqrt{\frac{1}{m_i}})$ 

\paragraph{Remark} The bound proposed in (\ref{wbound}) is analogue to the bound in the $\mathcal{H}$ divergence measure with completely different assumption. For example, the Wasserstein bound is derived on the real value output function, which can be naturally applied in the regression problem. 

\section*{Experiment details}
\subsection*{Digits recognition}
In the digit recognition, we used three different kinds of digits: Mnist \cite{lecun1998gradient}, MnistM \cite{arbelaez2011contour,ganin2016domain} and SVHN \cite{netzer2011reading}. As we described in the the paper, we only sample 3K, 5K and 8K examples for each task. The input image dimension is $28\times28$.

We used a modified LeNet-5 architecture for training the digit datasets.

\begin{itemize}
    \item Feature extractor: with 2 convolution layers. 
    
    'layer1': 'conv': [1, 32, 5, 1, 2], 'relu': [], 'maxpool': [3, 2, 0],
    
    'layer2': 'conv': [32, 64, 5, 1, 2], 'relu': [], 'maxpool': [3, 2, 0]
    
    \item Task prediction: with 2 fc layers.
    
     'layer3': 'fc': [*, 128], 'act\_fn': 'elu',
    
    'layer4': 'fc': [128, 10], 'act\_fn': 'softmax'
    
    \item Discriminator part:  with 2 fc layers.
    
    \emph{reverse\_gradient}()
    
    'layer3': 'fc': [*, 128], 'act\_fn': 'elu',
    
    'layer4': 'fc': [128, 10], 'act\_fn': 'softmax'

\end{itemize}

\paragraph{Hyper-parameter setting}
We set the $\rho= \frac{1}{T}$, with $T$ the number of the task. $\kappa_1 = 1$ and tuning the hyper-parameter $\kappa_2$ from $0.2$ to $2$ through grid search. In the Wasserstein-1 distance based approach, we set the gradient penalty weight as $1$.

As for the configurations for training the neural networks, we used \emph{SGD} optimizer with learning rate $0.01$ and momentum $0.9$. The maximum training epoch is $100$ for the proposed approach and baselines.

\subsection*{Amazon reviews}
We also evaluate the proposed algorithm in \emph{Amazon reviews} datasets. We extract reviews from four kinds of product (book, dvd disks, electronics and kitchen appliances). Reviews datasets are pre-processed with the same strategy from \cite{ganin2016domain}:  10K dimensional input features and binary output labels $\{0,1\}$, "0" if the product is ranked less equal than 3 stars, and "1" if higher than 3 stars. For each task we have $1000$ and $1600$ labelled training examples, respectively. 

We used the standard MLP architecture for training the pre-processed dataset.

\begin{itemize}
    \item Feature extractor: with 2 fc layers. 
    
    'layer1': 'fc': [10000, 256], 'act\_fn': 'elu',
    
    'layer2': 'fc': [256,128], 'act\_fn': 'elu', 
    
    \item Task prediction: with 2 fc layers.
    
     'layer3': 'fc': [128, 64], 'act\_fn': 'elu',
    
    'layer4': 'fc': [64, 1], 'act\_fn': 'sigmoid'
    
    \item Discriminator part:  with 2 fc layers.
    
    \emph{reverse\_gradient}()
    
    'layer3': 'fc': [128, 64], 'act\_fn': 'elu',
    
    'layer4': 'fc': [64, 1], 'act\_fn': 'sigmoid'
\end{itemize}

\paragraph{Hyper-parameter tuning}
We set the $\rho= \frac{1}{T}$, with $T$ the number of the task. $\kappa_1 = 1$ and tuning the hyper-parameter $\kappa_2$ from $0.2$ to $1$ through grid search. In the Wasserstein-1 distance based approach, we set the gradient penalty weight as $1$.

As for the configurations for training the neural networks, we used \emph{SGD} optimizer with learning rate $0.005$ and momentum $0.9$. The maximum training epoch is $100$ for the proposed approach and baselines.

\bibliographystyle{named}
\bibliography{reference}

\end{document}